\documentclass[sigconf]{acmart}
\usepackage{amsmath}
\usepackage{amsthm}
\usepackage{dsfont}
\usepackage{newtxmath}  

\usepackage{multirow}
\usepackage{booktabs}       
\usepackage{nicefrac}       

\usepackage{mathtools}
\usepackage{amsthm}
\usepackage{verbatim} 
\usepackage{subcaption}
\usepackage{algorithm}
\usepackage{algorithmic}
\usepackage{graphicx}
\usepackage{multirow}
\usepackage{hyperref}
\usepackage{array}
\usepackage[table]{xcolor}   
\definecolor{Lavender}{RGB}{230,230,250}  

\makeatletter
\@ifundefined{Bbbk}{}
{}
\makeatother

\newcommand{\subsidyrate}{\sigma}     
\newcommand{\exposureth}{\kappa}      
\newcommand{\feerate}{f}              
\newcommand{\feeth}{\kappa_f}         
\newcommand{\platformspend}{\rho_t} 
\newcommand{\absspend}{R_t}         
\newcommand{\userinvest}{p}           
\newcommand{\useractivity}{h}         

\AtBeginDocument{%
  }

\begin{document}

\title{Invariant Causal Routing for Governing Social Norms in Online Market Economies}

\author{Xiangning Yu}
\affiliation{%
  \institution{Tianjin University}
  \city{Tianjin}
  \country{China}
}

\author{Qirui Mi}
\affiliation{%
  \institution{Institute of Automation, Chinese Academy of Sciences}
  \city{Beijing}
  \country{China}
}

\author{Xiao Xue}
\authornote{Corresponding authors.}
\affiliation{%
  \institution{Tianjin University}
  \city{Tianjin}
  \country{China}
}

\author{Haoxuan Li}
\affiliation{%
  \institution{Peking University}
  \city{Beijing}
  \country{China}
}

\author{Yiwei Shi}
\affiliation{%
  \institution{University of Bristol}
  \city{Bristol}
  \country{United Kingdom}
}

\author{Xiaowei Liu}
\affiliation{%
  \institution{Tianjin University}
  \city{Tianjin}
  \country{China}
}

\author{Mengyue Yang}
\authornotemark[1]
\affiliation{%
  \institution{University of Bristol}
  \city{Bristol}
  \country{United Kingdom}
}

\renewcommand{\shortauthors}{Yu et al.}


\begin{abstract}
Social norms are stable behavioral patterns that emerge endogenously within economic systems through repeated interactions among agents. In online market economies, such norms—like fair exposure, sustained participation, and balanced reinvestment—are critical for long-term stability. We aim to understand the causal mechanisms driving these emergent norms and to design principled interventions that can \textit{steer} them toward desired outcomes. This is challenging because norms arise from countless micro-level interactions that aggregate into macro-level regularities, making causal attribution and policy transferability difficult. To address this, we propose \textbf{Invariant Causal Routing (ICR)}, a causal governance framework that identifies policy–norm relations stable across heterogeneous environments. ICR integrates counterfactual reasoning with invariant causal discovery to separate genuine causal effects from spurious correlations and to construct \textit{interpretable, auditable policy rules} that remain effective under distribution shift. In heterogeneous agent simulations calibrated with real data, ICR yields more stable norms, smaller generalization gaps, and more concise rules than correlation or coverage baselines, demonstrating that causal invariance offers a principled and interpretable foundation for governance. 
\end{abstract}


\maketitle

\addtocontents{toc}{\protect\setcounter{tocdepth}{-1}}

\section{Introduction}
\label{sec:intro}

Social norms are commonly defined as stable patterns of behavior that emerge within groups through repeated interactions. They delineate acceptable individual actions while sustaining social order and cooperation at the collective level \cite{bicchieri2005grammar, fehr2000cooperation}. Their binding force rests on mutual expectations: individuals comply because they anticipate others will do the same \cite{lewis2008convention}. A defining property of norms is therefore their \textit{stability}---they persist over long time horizons, withstand environmental shocks, and display consistent adherence across groups and contexts. We focus on understanding how social norms arise in economic systems and how governance actors can intentionally \textit{steer} the formation of social norms toward desirable collective outcomes.  

However, the formation of social norms in large-scale adaptive systems is far from transparent. Countless local interactions aggregate into macro-level regularities \cite{lewis2008convention}, making it difficult to attribute outcomes to specific actions or interventions. This opacity becomes critical when the governance actor (the online market) seeks to steer norm formation through intervention. Governance rarely dictates behavior directly; instead, it shapes \textit{structural conditions}—such as subsidies, fee rates, and exposure or pricing rules—to elicit self-organized responses that crystallize into collective norms. However, interventions often yield divergent results, effective in one context yet ineffective in another, because outcomes depend less on implementation quality than on behavioral cues and confounders including visibility, sanction, reciprocity, and imitation.

These difficulties highlight several research challenges. First, norms emerge from decentralized interactions, making causal responsibility opaque. Second, interventions face multistability and path dependence, as identical policies may lead to divergent equilibria under different initial conditions. Third, heterogeneity and asymmetry in populations amplify disparities: the same cue may elicit different reactions across groups, and uneven visibility or sanctioning further skews outcomes. Finally, even when structural relations remain intact, distributional shifts in confounders or exogenous variables distort correlations, weaken extrapolation, and raise the cost of trial-and-error.

\begin{figure}[t]
  \centering
\includegraphics[width=0.4\textwidth]{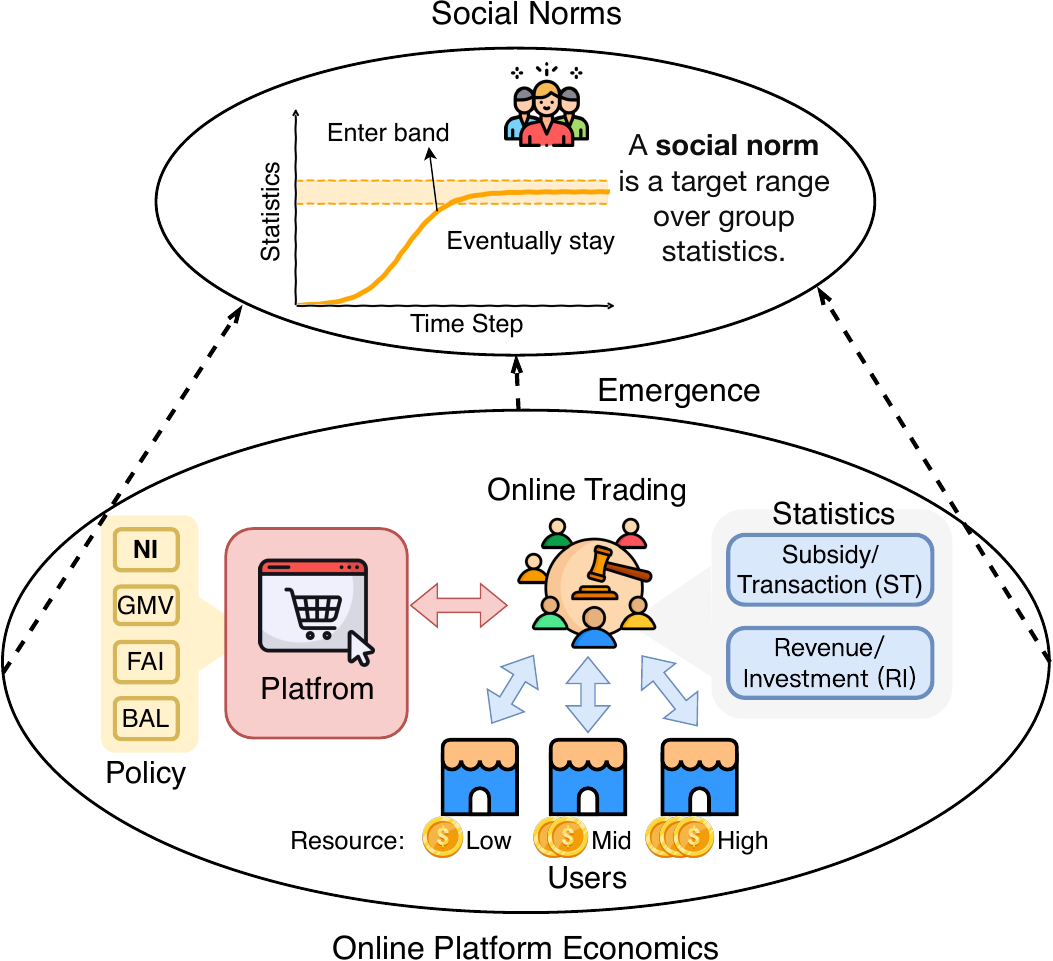}
 \caption{Social norms in onlineeconomies. }
\label{fig:social_norm}
\vspace{-0.45cm}
\end{figure}

Existing research has primarily emphasized natural emergence or aggregate outcomes in multi-agent coordination, focusing on what works ``on average.'' However, such approaches do not address the governance-relevant question: \textit{why does a given intervention succeed in one context but fail in another?} If distributional drift is the root cause, the remedy lies not in increasingly complex correlational modeling but in uncovering \textbf{causal structures invariant across environments}.  

This paper advances the claim that governance should prioritize identifying \textbf{invariant causal factors} that determine success or failure under distributional shift, rather than comparing surface-level strategies. To this end, we adopt \textbf{causal invariance} as a guiding principle and employ the \textbf{Probability of Necessity and Sufficiency (PNS)} to provide counterfactual evidence. PNS formalizes success as occurring \textit{if and only if} a given intervention is applied, thereby enabling the identification of \textbf{invariant causal routes} that remain effective across initial conditions, perturbations, and random seeds. Compared with correlation-based heuristics, PNS yields stronger out-of-distribution robustness and interpretability.  

We define \textit{norm achievement} as a threshold event where key macro statistics (e.g., revenue-to-investment and subsidy-to-transaction ratios) enter and persist within a target range, capturing long-term stability while tolerating short-term fluctuations. To operationalize this idea, we introduce a three stage causal framework that (i) identifies policies effective only when applied, revealing those invariant across environments; (ii) compiles them into a concise rule table of the form ``if context $\psi$, then apply intervention $\theta$,'' emphasizing robustness and minimal redundancy; and (iii) traces the causal pathway from governance actions to fundamental factors and eventual norm achievement, providing interpretable explanations. 

We validate this framework in an online marketplace with heterogeneous users, where subsidies and exposure policies vary while user strategies remain fixed. Norm achievement is assessed by whether macro statistics enter and persist within the target range, allowing direct evaluation of governance effectiveness.

This paper advances a causal perspective on governance by moving beyond average-effect optimization toward discovering invariant causal relations that determine policy success under distributional shifts. Our main contributions are:

\begin{enumerate}
    \item \textbf{Algorithmic Innovation – Invariant Causal Routing (ICR).} 
    We propose a novel three-stage framework that integrates causal inference and rule-based policy learning. ICR identifies \textit{causal routes}—policy-to-norm relations that remain valid across heterogeneous environments—and formulates them into a minimal, auditable rule list.

    \item \textbf{Interpretability and Causal Accountability.} 
    Grounded in the Probability of Necessity and Sufficiency (PNS), ICR provides transparent causal guarantees, distinguishing genuine policy effects from confounded correlations and offering human-readable explanations of causal mechanisms.

    \item \textbf{Empirical Validation under Distribution Shift.} 
    Through heterogeneous-agent simulations calibrated with real-world data, ICR maintains stable norm attainment under out-of-distribution conditions, achieving smaller generalization gaps and higher robustness than correlation-based and heuristic baselines.
\end{enumerate}

Together, these contributions show that \textbf{causal invariance} offers a principled basis for designing stable, interpretable, and transferable governance strategies in complex socio-economic systems.

\section{Related Work}
\label{sec:related}

The related literature spans three complementary strands.
\paragraph{Social Norms.}
Research on social norms examines their spontaneous emergence without central control, explained through expectation--compliance frameworks~\cite{bicchieri2005grammar}, stochastic evolutionary games~\cite{young1993evolution,yang2025twinmarket}, and multi-agent social laws~\cite{shoham1995social,shoham1997emergence}. Social psychology and experimental economics study their definition, maintenance, and measurement~\cite{krupka2013identifying,opp1982evolutionary,berkowitz2005overview,hechter2001social}, while recent LLM-based work shows that linguistic interaction can itself yield conventions and biases~\cite{hawkins2019emergence,dai2024artificial}. Yet prior studies lack (i) operational criteria for detecting norm formation in long-term data and (ii) causal tools linking interventions to macro norms~\cite{bicchieri2005grammar,young2015evolution,krupka2013identifying,sen2007emergence,opp1982evolutionary,chung2016social,shaffer1983toward,hechter2001social,dai2024artificial}. We define norm attainment as system-level indicators stabilizing within a tolerance band, enabling causal analysis along the chain intervention $\rightarrow$ micro responses $\rightarrow$ macro norms.
\paragraph{From Correlational to Causal Policy Governance.}
Conventional policy evaluation relies on correlations with limited causal interpretability. A/B testing and observational designs estimate average treatment effects~\cite{dellavigna2018motivates,angrist2009mostly}, while offline estimators such as propensity scores, IPW, and DR~\cite{rosenbaum1983central,dudik2011doubly}, and macro designs like difference-in-differences or synthetic control~\cite{abadie2010synthetic}, struggle with heterogeneity~\cite{imbens2015causal,athey2017state,pearl2009causality,wager2018estimation,kunzel2019metalearners} and distribution shifts~\cite{bareinboim2016causal,peters2016causal,hernan2010causal}. They often yield broad rather than context-specific insights~\cite{dellavigna2018motivates,athey2021policy}, while offline evaluation faces bias–variance and robustness challenges~\cite{jiang2016doubly,thomas2016data}. Deep RL, bandit, and learning-to-rank methods improve short-term outcomes~\cite{dulac2019challenges,swaminathan2015counterfactual,schnabel2016recommendations,joachims2017unbiased} but remain opaque and hard to audit~\cite{rudin2019stop}. Bridging interpretability, stability, and causal validity thus remains an open challenge.

\paragraph{OOD Prediction and PNS}
PNS formalizes the causal responsibility of an intervention in potential-outcome terms.  
Given treatment $A\!\in\!\{0,1\}$ and potential outcomes $Y(A)$,  
$\mathrm{PNS}=\Pr\{Y(1)=1,\;Y(0)=0\}$ measures the probability that an event occurs only under treatment~\citep{pearl2009causality,tian2000probabilities}.  
Causal approaches to out-of-distribution (OOD) generalization assume domain-invariant mechanisms.  
Some explicitly constrain causal graphs~\citep{pfister2019invariant,rothenhausler2021anchor,heinze2021conditional,gamella2020active,oberst2021regularizing,zhang2015multi},  
while others treat invariance as a representation-learning objective, such as IRM and its extensions in game-theoretic, variance-penalized, and nonlinear settings~\citep{ahuja2020invariant,krupka2013identifying,lu2021invariant,gulrajani2020search}.  
Unlike these constraints, \textit{PNS-based invariant learning}~\citep{yang2023invariant} integrates PNS directly into training, minimizing necessity–sufficiency mismatch via paired counterfactuals to yield representations that are both causally decisive and stable across domains.

\section{Preliminaries and Problem Formulation}
\label{sec:problem}

\subsection{Social Norms in Online Market}
\label{sec:socialnorm}

Social norms are stable, self-enforcing behavioral regularities that emerge endogenously through repeated interactions among heterogeneous agents. In online market economies, such norms manifest as persistent collective patterns—\textbf{fair exposure}, \textbf{sustained participation}, or \textbf{balanced reinvestment}—rather than externally imposed rules. They are characterized not by precise equilibria, but by long-run \emph{bands} of system-level indicators (e.g., exposure share, activity level, reinvestment ratio) that remain within tolerance ranges over time.

In this work, we treat social norm formation as a measurable, long-term event: a system achieves a norm when its trajectory statistics enter and \emph{persist} within a specified tolerance band. This operationalization allows for direct causal analysis and accommodates short-term fluctuations.

Formally, let $\phi:\mathcal{X}\!\to\!\mathbb{R}^d$ denote bounded continuous statistics (e.g., group-wise Revenue/Investment ratio or  Subsidy/Transaction ratio). After a burn-in period $T_{\mathrm{burn}}$, the window-$T$ average is:
\begin{equation}
\bar{\phi}_T = \frac{1}{T}\sum_{t=T_{\mathrm{burn}}+1}^{T_{\mathrm{burn}}+T}\phi(X_t),
\end{equation}
where $\phi(X_t)$ captures $d$-dimensional group metrics at time $t$.

\begin{definition}[Social Norm Band]
Given a target vector $\eta\!\in\!\mathbb{R}^d$ and tolerance $\varepsilon\!\in\!(0,\infty)^d$, the \emph{social norm band} is
\begin{equation}
\mathcal{S}_\varepsilon = \prod_{j=1}^d [\eta_j-\varepsilon_j,\ \eta_j+\varepsilon_j].
\end{equation}
A run under policy $\theta$ \emph{attains the norm} if there exists $T_0<\infty$ such that $\bar{\phi}_T\in\mathcal{S}_\varepsilon$ for all $T\!\ge\!T_0$.
\end{definition}

\begin{definition}[Social norm attainment]
\label{def:norm-stay}
Under platform strategy $\theta$, a run \emph{attains a social norm} if 
\begin{equation}
\exists\, T_0 < \infty \quad \forall\, T \ge T_0: \quad \bar\phi_T \in \mathcal S_\varepsilon.
\end{equation}
In other words, the time-averaged statistics eventually stay within the band, and exact convergence to a single point is unnecessary.
\end{definition}

\paragraph{Distance convention.}
We use the sup norm $\|v\|_{\infty}=\max_j |v_j|$ on $\mathbb{R}^d$ and the induced distance to a set:
\begin{equation}
\mathrm{dist}_\infty(x,S)\;=\;\inf_{y\in S}\,\|x-y\|_\infty.
\end{equation}
For the rectangular band $\mathcal S_\varepsilon=\prod_{j=1}^d[\eta_j-\varepsilon_j,\eta_j+\varepsilon_j]$, this reduces to
\begin{equation}
\mathrm{dist}_\infty(x,\mathcal S_\varepsilon)
=\max_{1\le j\le d}\Big\{\max\big\{(\eta_j-\varepsilon_j)-x_j,\; x_j-(\eta_j+\varepsilon_j),\;0\big\}\Big\}.
\end{equation}

\paragraph{Existence of feasible limits.}  
Under mild and standard assumptions (compactness and Feller continuity), Lemma~\ref{lem:invariant} (App.~\ref{app:invariant}) ensures the set of invariant measures $\mathcal I(P^{(\theta)})$ is nonempty and compact. Consequently, the \emph{asymptotic signature set}
\begin{equation}
\label{eq:sigma}
\Sigma(P^{(\theta)},\phi) = \left\{\int \phi\,d\pi: \pi \in \mathcal I(P^{(\theta)}) \right\} \subset \mathbb{R}^d
\end{equation}
is nonempty and compact, and every subsequential limit of $(\bar\phi_T)$ lies in $\Sigma(P^{(\theta)},\phi)$ by Theorem~\ref{thm:C-sigma} (App.~\ref{app:occ}). Thus, $\Sigma$ organizes the feasible long-run profiles of $(\bar\phi_T)$.

\begin{proposition}[Norm Feasibility Criterion]
\label{prop:feasibility}
If $\Sigma(P^{(\theta)}, \phi) \cap \mathcal S_\varepsilon \neq \varnothing$, the social norm is \emph{feasible}, meaning there exists a stationary regime whose averages satisfy the band.  
Conversely, if $\Sigma(P^{(\theta)}, \phi) \cap \mathcal S_\varepsilon = \varnothing$, then there exists a constant $\delta > 0$ such that
\begin{equation}
\liminf_{T \to \infty} \mathrm{dist}_\infty \left( \bar\phi_T, \mathcal S_\varepsilon \right) \ge \delta \quad \text{almost surely}.
\end{equation}
That is, stable attainment is impossible if all invariant averages lie outside the band.
\end{proposition}

The existence result makes the event well-posed; feasibility depends on whether the band intersects $\Sigma(P^{(\theta)}, \phi)$. We do not require $\bar{\phi}_T$ to converge to a single limit—``eventual stay'' suffices. Further discussion on convergence and ergodicity appears in App.~\ref{app:conv}.

\subsection{System Setup and Dynamics}
\label{sec:setup}

We consider a platform ecosystem comprising a platform and multiple heterogeneous user groups that repeatedly interact over time. The platform follows a fixed objective $\theta\!\in\!\Theta$ and influences user behavior through policy levers such as subsidies, fee rates, and exposure thresholds. The joint state of the system evolves as a homogeneous Markov chain:
\begin{equation}
\{X_t\}_{t\ge0}, \qquad X_{t+1}\!\sim\!P^{(\theta)}(\cdot\,|X_t),
\end{equation}
on a compact state space $\mathcal{X}$.  
Each $\theta$ corresponds to a distinct governance regime—e.g., GMV growth (GMV), fairness (FAI), balanced growth–fairness (BAL), or user welfare/retention (UW)—with a normal-intervention (NI) baseline as reference.

\begin{figure*}[t]
  \centering
\includegraphics[width=0.8\textwidth]{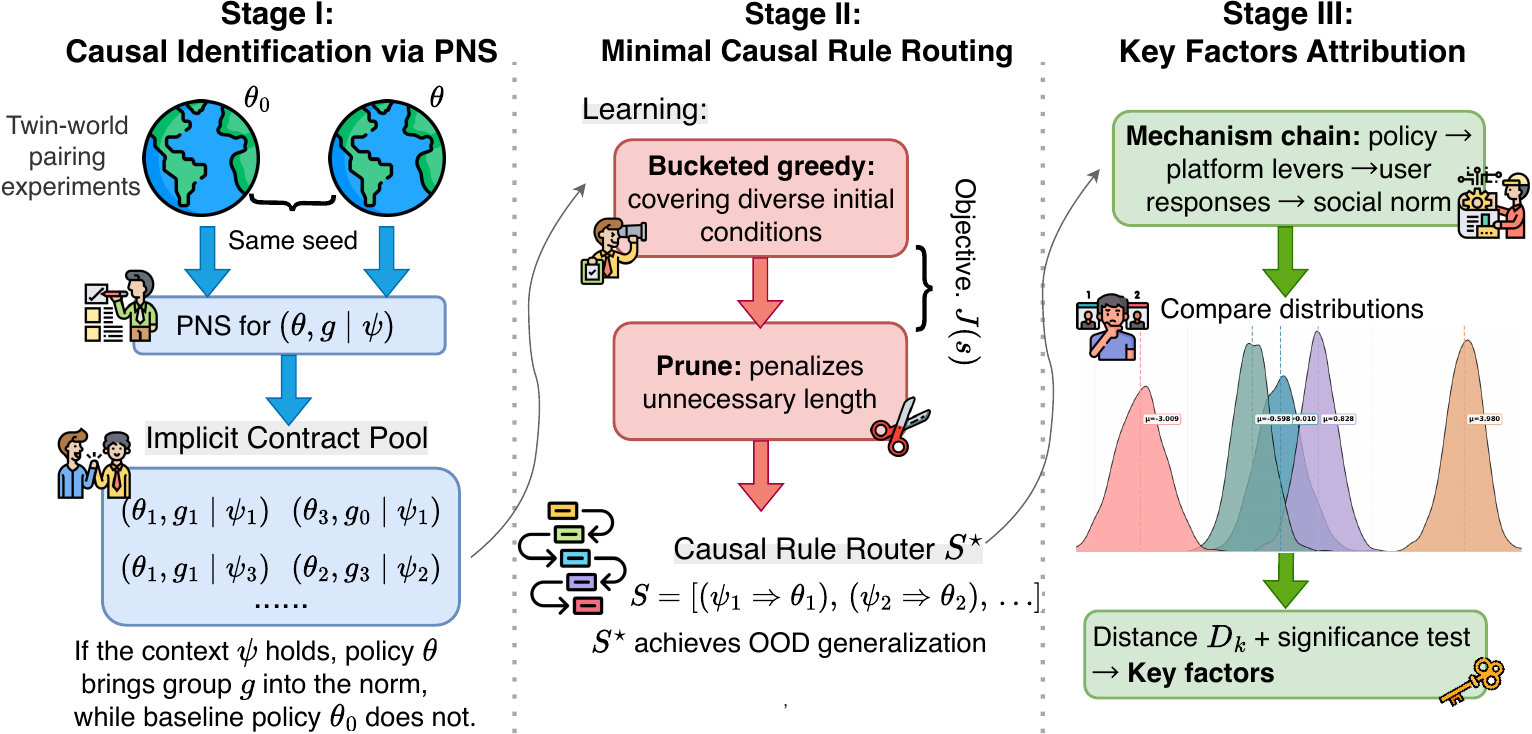}
\caption{Three-stage framework for discovering causal strategies in social norm formation.
Stage~I tests causal effects by identifying implicit contracts $(\theta, g \mid \psi)$ where policy $\theta$ enables group $g$ to reach the norm under context $\psi$.
Stage~II learns a compact rule router $S^\star$ selecting effective strategies for robust norm attainment.
Stage~III explains success by attributing norms to key factors where user responses differ from the baseline $\theta_0$.}
\label{fig:main}
\end{figure*}

\subsection{Problem Definition}
\label{sec:problem_def}

The governance objective is to identify causal mappings from platform policies $\theta$ to long-run social-norm outcomes and determine which mappings remain \emph{invariant} across heterogeneous environments or distributional shifts.  
This reframes governance from \emph{average-effect optimization} to \emph{causal invariance}—discovering factors and routes that \emph{consistently cause} norm attainment regardless of confounders. We therefore investigate three key research questions:

\begin{enumerate}
  \item \textbf{RQ1 (What: causal impact).} How do different platform objectives \emph{causally influence} the emergence and stabilization of social norms across user groups?
  \label{rq1}

  \item \textbf{RQ2 (How: invariant causal routing).} Under changing or out-of-distribution (OOD) conditions, what is the \emph{shortest and most stable routing policy}—a compact mapping from context to action—that maintains consistent causal effects on norm attainment?
  \label{rq2}

  \item \textbf{RQ3 (Why: causes of norm divergence).} What underlying \emph{mechanisms and structural factors} explain why different groups or contexts evolve toward distinct social norms, and which are necessary or sufficient for these divergences?
  \label{rq3}
\end{enumerate}

\section{Method: Three-Stage Causal Governance}
\label{sec:method}

 We propose a three-stage causal governance framework addressing RQ1--RQ3: 
\textbf{ Stage (I) Causal identification} estimates probabilistic necessity–sufficiency (PNS) to test whether a strategy $\theta$ causes a group $g$ to attain its norm band under context $\psi$. 
\textbf{Stage (II) Invariant causal routing} learns a minimal first-match rule list mapping contexts to actions, maximizing norm attainment under distribution shift while ensuring causal stability. 
\textbf{Stage (III) Key factors attribution} contrasts $\theta$ with its baseline to reveal structural and behavioral factors that drive norm stabilization or divergence.

\subsection{Stage I: Causal Identification via PNS}
\label{sec:stage1}

We define when a platform strategy $\theta$ is \emph{causally responsible} for a group $g$ meeting its target band, relative to a baseline $\theta_0$, through the \textbf{Probability of Necessity and Sufficiency (PNS)}.

\begin{definition}[Potential Outcome and PNS]
  Let \( Y_g(\theta) = B^{(g)}(\theta) \in \{0,1\} \) denote the potential outcome for group \(g\) under intervention \(\theta\), where \( B^{(g)}(\theta) = 1 \) indicates the group meets the target band. Following Pearl's counterfactual framework~\cite{pearl2009causality}, for a context predicate \(\psi\) on the initial state, the Probability of Necessity and Sufficiency (PNS) is the probability that replacing the baseline strategy \(\theta_0\) with \(\theta\) enables group \(g\) to enter the norm band.

\begin{equation}
\label{def:pns}
\mathrm{PNS}_g(\theta \mid \psi) = \Pr\left(Y_g(\theta) = 1, Y_g(\theta_0) = 0 \mid \psi = 1\right).
\end{equation}
This probability captures both the \textit{necessity} (group does not meet the target under $\theta_0$) and \textit{sufficiency} (group meets the target under $\theta$) of the platform strategy.
\end{definition}

\paragraph{PNS Identification}
We estimate PNS using paired runs on the same seed. No additional exogeneity or monotonicity assumptions are needed due to the paired design. 
We test whether a strategy $\theta$ is \emph{causally responsible} for enabling group $g$ to attain its social-norm band under a given context $\psi$.


\paragraph{Estimation.}  
Because both potential outcomes are observed under paired runs, PNS is point-identified by the unbiased estimator
\begin{equation}
\widehat{\mathrm{PNS}}_g(\theta \mid \psi) 
= \frac{1}{N_\psi} \sum_{e:\psi(X_e)=1} 
\mathbf{1}\!\left\{ Y_{g,e}(\theta)=1,\; Y_{g,e}(\theta_0)=0 \right\},
\end{equation}
where $N_\psi$ is the number of seeds with $\psi(X_e)=1$.  
Confidence intervals see App.~\ref{app:CI}.

\begin{definition}[Implicit Contract]
\label{def:implicit-contract}
Given thresholds $n_{\min}\in\mathbb N$ and $\tau_{\mathrm{pns}}\in[0,1]$, we call $(\theta,g\mid\psi)$ an \emph{implicit contract} if it satisfies both:
\begin{align}
& \text{(support)} \quad N_\psi \;\ge\; n_{\min}, \\
& \text{(accountability)} \quad \widehat{\mathrm{PNS}}_g(\theta\mid\psi) \;\ge\; \tau_{\mathrm{pns}},
\end{align}
i.e., the clause has sufficient support across initial conditions and its PNS confidence exceeds the target level.
\end{definition}

\noindent\textbf{Interpretation.}  
An implicit contract is a reliable “if–then” promise in the social-norm sense:  
\emph{if the context $\psi$ holds for an initial condition, then applying $\theta$ is causally sufficient to bring group $g$ into norm compliance, whereas baseline behavior would fail.}

\paragraph{Outcome.}  
Stage I thus yields a sound pool of \emph{implicit contracts} $(\theta,g\mid\psi)$  that satisfy both support and accountability, forming the candidate set for Stage~II.

\subsection{Stage II: Minimal Causal Rule Routing}
\label{sec:stage2}

From Stage~I we obtain a set of \emph{implicit contracts} $(\theta,g\mid\psi)$ with certified PNS. 
We assemble them into a compact, ordered decision list:
\begin{equation}
S=[(\psi_1 \Rightarrow \theta_1),\,(\psi_2 \Rightarrow \theta_2),\,\dots,(\psi_{|S|} \Rightarrow \theta_{|S|})],
\end{equation}

which applies the \emph{first} matching rule to an episode with context $X_0$; otherwise we fall back to $\theta_0$.

\paragraph{Covering diverse initial conditions (bucketed scoring).}
To explicitly expand coverage over heterogeneous initial conditions, we partition the initial contexts into buckets $\mathcal{B}=\{b\}$ (e.g., by quantiles or clustering on $X_0$) with empirical weights $w_b$. 
For each candidate rule $i$ we compute its \emph{bucketed marginal coverage} $m_{i,b}$ (the fraction of bucket-$b$ episodes newly covered under first-match order) and its certified causal gain ${\rm PNS}_{g,i,b}$ per group $g$ within bucket $b$.
This favors rules that both \emph{reach new initial-condition mass} and \emph{deliver certified gains} where they apply.

\paragraph{Objective.}
We maximize a bucketed objective that rewards union coverage of initial conditions and penalizes unnecessary length:
\begin{equation}
J(S)=\sum_{b\in\mathcal{B}} w_b \sum_{i=1}^{|S|} m_{i,b}\cdot\Big(\sum_g w_g\,{\rm PNS}_{g,i,b}\Big)\;-\;\lambda\,|S|.
\end{equation}
Here $m_{i,b}$ accounts for first-match semantics (only \emph{new} coverage counts), ensuring the list grows coverage across buckets rather than overfitting a few.

\paragraph{Learning: bucketed greedy + prune.}
We learn a short first-match rule list on validation environments as follows: (i) rank candidates by $\sum_g w_g\,{\rm PNS}_{g,i}$; (ii) iteratively add the rule with the largest \emph{bucketed marginal} improvement in $J(S)$—only not-yet-covered contexts count—and enforce a \emph{coverage safeguard} $\min_b \mathrm{Cov}_b(S)\!\ge\!\tau_{\rm cov}$; (iii) run a backward prune that removes any rule whose deletion reduces $J$ by at most $\tau_{\text{prune}}$. The resulting $S^\star$ covers diverse initial-condition buckets with few rules; we then evaluate $S^\star$ on held-out \texttt{test} seeds. See Algorithm~\ref{alg:learnrouter}.

\paragraph{Outcome.}
Stage~II yields a parsimonious router $S^\star$ with invariant causal effects. With PNS-filtered clauses and stable rule gains across buckets, $S^\star$ generalizes reliably under distribution shift.

\begin{algorithm}[t]\small
\caption{\textsc{LearnRouter} (Greedy + Buckets + Prune)}
\label{alg:learnrouter}
\begin{algorithmic}[1]
\REQUIRE Candidates $C=\{(\psi,\theta)\}$; splits train/val/test; buckets $B$ with weights $w_b$; penalty $\lambda$; coverage threshold $\tau_{\mathrm{cov}}$; prune tol. $\tau_{\text{prune}}$; max rules $K$
\ENSURE Ordered router $S$
\STATE $S\gets[\,]$;\ for $b\in B$: $\mathrm{Cov}[b]\gets\emptyset$ \textit{(first-match covered set)}
\WHILE{$|S|<K$}
  \FOR{each $c\in C\setminus S$}
    \STATE $\mathrm{ncov}[c]\gets\sum_{b} w_b\cdot\mathrm{NewCov}(c,b,\mathrm{Cov}[b];\text{train})$
    \STATE $\mathrm{gain}[c]\gets\sum_{b} w_b\cdot\mathrm{Gain}(c,b;\text{val})$
    \STATE $\mathrm{score}[c]\gets \mathrm{ncov}[c]\cdot \mathrm{gain}[c]-\lambda$
  \ENDFOR
  \STATE $c^\star\gets\arg\max_c \mathrm{score}[c]$;\ \IF{$\mathrm{score}[c^\star]\le0$}\STATE\textbf{break}\ENDIF
  \STATE append $c^\star$ to $S$;\ update all $\mathrm{Cov}[b]$
  \IF{$\min_{b}\mathrm{BucketCov}(S,b;\text{val})<\tau_{\mathrm{cov}}$}
    \STATE add best remaining rule that raises uncovered-bucket coverage with positive $\mathrm{score}$
  \ENDIF
\ENDWHILE
\FOR{each rule $r$ in $S$ from last to first}
  \IF{$\mathrm{Val}(S\setminus\{r\};\text{val})\ge \mathrm{Val}(S;\text{val})-\tau_{\text{prune}}$}
    \STATE remove $r$
  \ENDIF
\ENDFOR
\STATE \RETURN $S$
\end{algorithmic}
\end{algorithm}

\subsection{Stage III: Key Factors Attribution}
\label{sec:stage3}

Having identified effective contracts and assembled a rule router, we now seek to explain \emph{why} certain strategies succeed while others fail under the \emph{same initial conditions}. Stage~III isolates the mechanism factors that differ when the PNS event---success under $\theta$ but failure under $\theta_0$---occurs.

\paragraph{Matched comparison.}
For an implicit contract $(\theta,g\mid\psi)$, we focus on the same set of episodes satisfying $\psi(X_e)=1$, representing comparable starting conditions for group $g$. Within this matched set, we contrast the system's internal responses under the target and baseline policies. Formally, for each lever factor $f_k$,
\begin{equation}
S_k^{(\theta)} = \{\, f_k(e,\theta) : \psi(X_e)=1 \,\}, 
\qquad
S_k^{(\theta_0)} = \{\, f_k(e,\theta_0) : \psi(X_e)=1 \,\}.
\end{equation}
This captures how the same initial state evolves differently under alternative interventions.

\paragraph{Distributional divergence.}
We quantify the induced change on each factor using a normalized distance metric
\begin{equation}
D_k = \mathrm{Dist}\big(S_k^{(\theta)}, S_k^{(\theta_0)}\big),
\end{equation}
where $\mathrm{Dist}$ can be Wasserstein-1. The metric, normalization, and estimation details are provided in App.~\ref{app:hparams}.

\paragraph{Identifying key factors.}
A factor $f_k$ is deemed \emph{key} if (i) $D_k \ge \theta_{\mathrm{dist}}$, and (ii) the difference is statistically significant under the test and correction procedures described in App.~\ref{app:hparams}. These factors represent the levers through which $\theta$ causally alters user or platform behavior relative to $\theta_0$.

\paragraph{Outcome.}
Stage~III yields ranked key factors that separate successful from failed interventions under identical conditions, offering interpretable causal levers for understanding and transferring social-norm formation.

\section{Experiments}
\label{sec:experiments}

We evaluate the proposed three-stage causal governance framework in a simulated online marketplace \emph{ecosystem} calibrated with real-world data from the 2022 \emph{Survey of Consumer Finances} (SCF) \cite{scf2022}. To address RQ~\ref{rq1}, we estimate PNS via twin-world pairing to quantify the causal impact of platform interventions on norm attainment. For RQ~\ref{rq2}, we learn a minimal first-match router $S^\star$ that generalizes across seeds and initial economic regimes, verifying causal invariance under distribution shift. For RQ~\ref{rq3}, we perform mechanism attribution by contrasting endogenous platform levers and user responses across causal routes to reveal why certain interventions induce or fail to sustain stable social norms.

\subsection{Experimental Setup}

\paragraph{Platform Intervention Goals.}
We consider five operational objectives for the online market: (\textbf{NI}) a fixed normal-intervention baseline; (\textbf{GMV}) maximizing transactions or gross merchandise volume growth; (\textbf{FAI}) improving exposure fairness to ensure equal opportunity across user groups; (\textbf{BAL}) pursuing a balanced objective combining GMV growth and fairness; and (\textbf{UW}) maximizing long-run ecosystem welfare and user retention.

\paragraph{Agent Heterogeneity.}
The ecosystem comprises a \textit{platform agent} and multiple \textit{user agents}. The platform adjusts \emph{subsidies}, \emph{fee/take-rate tiers}, \emph{exposure thresholds}, and off–transaction spending, while users choose their \emph{investment share} and \emph{activity level}. Users are grouped by resources (\textit{low}, \textit{mid}, \textit{high}), forming four distinct behavioral strategies within the system.

\paragraph{Initial Conditions}
To cover diverse regimes, we evaluate five initial configurations (IC1--IC5): baseline equilibrium, conservative, aggressive, balanced-liquidity, and robust-inequality economies (parameters in App.~\ref{app:setup}).

\paragraph{Seed Shift as OOD}
Unless otherwise noted, the structural equations and noise distribution are fixed across runs. Different random seeds are treated as a weak out-of-distribution (OOD) shift in the realization of exogenous noise and initial conditions, following $X_{t+1}=f_\theta(X_t,\xi_t)$ with $\xi_t\sim P(\xi)$. Training, validation, and test sets use disjoint seeds, while twin-world comparisons share the same noise realization within each seed $e$ for $(\theta,\theta_0)$.

\begin{figure*}[t]
  \centering
 \caption{Last-50-year trajectories of $\phi_1$ (ST) and $\phi_2$ (RI) with norm bands (per objective $\times$ group). Shaded regions of the same color indicate the band within which individuals of the corresponding group are expected to remain under the associated social norm. When all trajectories lie within their respective same-colored shaded regions, the social norm is fully attained.}
\label{fig:exp1_timeseries}
\includegraphics[width=0.9\textwidth]{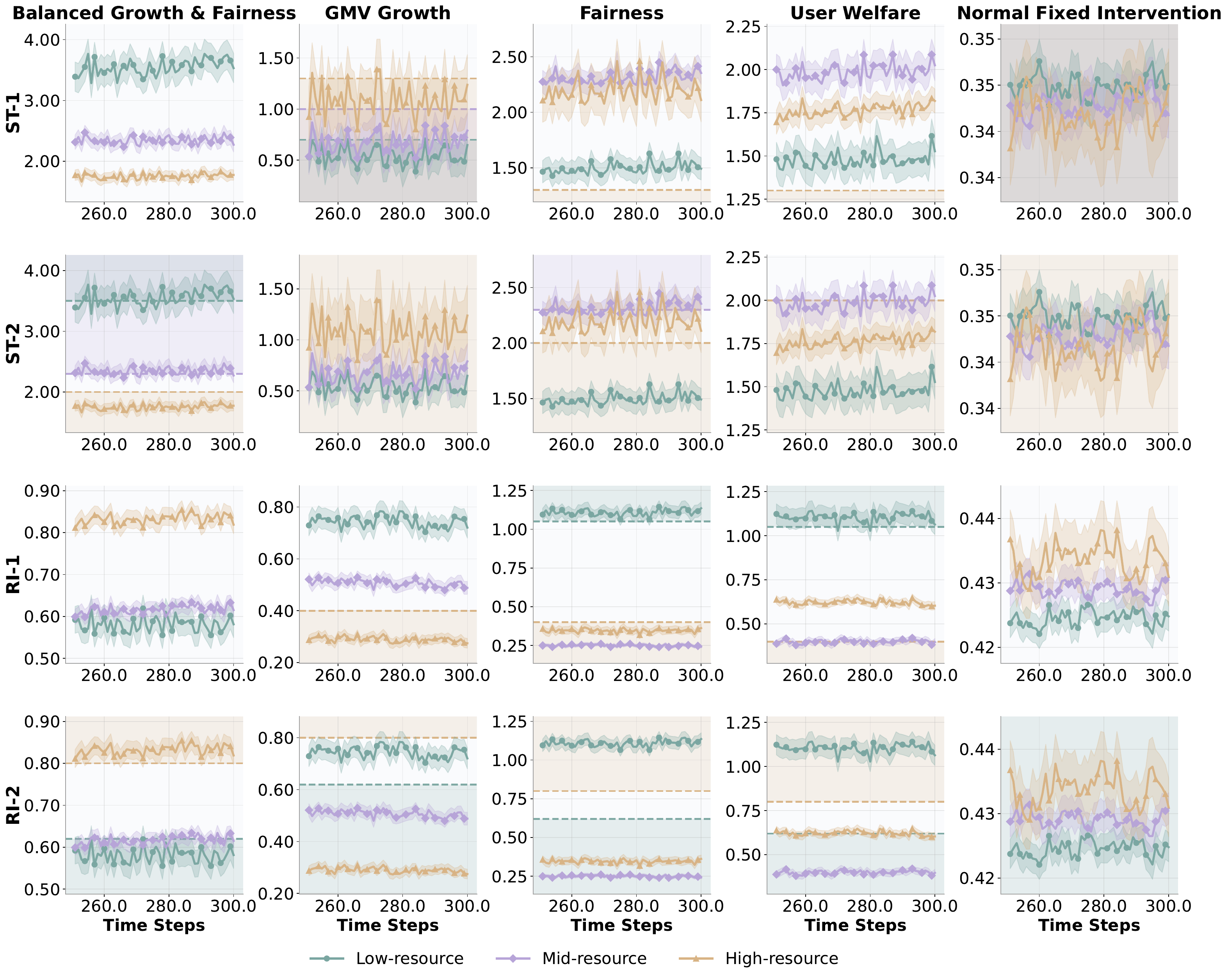}
\end{figure*}

\paragraph{Target Social Norms.}
We evaluate four social norms defined by two aggregate ratios: Subsidy/Transaction (ST) and Revenue/Investment (RI).  
For ST, \textbf{ST-1} represents low subsidy intensity where users rely mainly on organic traffic (laissez-faire), and \textbf{ST-2} denotes a regressive pattern—higher subsidies for low-resource users and lower for high-resource ones—stabilizing early growth but risking over-correction.  
For RI, \textbf{RI-1} indicates upward mobility, where low-resource users achieve higher ROI than incumbents, while \textbf{RI-2} captures concentration or entrenchment, where top users compound advantage and newcomers struggle to accumulate.
Precise social norm band ranges are given in App.~\ref{app:setup}.

\paragraph{Training Methodology}
Agents are trained with \textit{MADDPG} in a continuous, interdependent multi-agent environment with behaviors (investment/effort) and dynamic platform levers (subsidies/fees/exposure). Detailed hyperparameters are in App.~\ref{app:train}.

\paragraph{Evaluation Statistics.}
We measure norm emergence using two trajectory statistics $\phi_j(X_t)$: $\phi_1$ (Subsidy-to-Transaction ratio, ST) capturing platform subsidy intensity, and $\phi_2$ (Revenue-to-Investment ratio, RI) as a group-wise ROI proxy. After burn-in, we compute the window average $\bar{\phi}_{j,T}$ over the last $T{=}50$ steps (one step = one year). 

\paragraph{Ablation Study Baselines}

We train routers on \texttt{train} seeds and evaluate on disjoint \texttt{test} seeds across IC1--IC5; seed splits and baseline configurations are in App.~\ref{app:ablation-seeds}.
Our baselines:
\begin{itemize}
\item \textbf{PNS+Greedy}: Stage~I supplies PNS-certified clauses; rank candidates by mean target-conjunction PNS on \texttt{train}, add rules greedily, and stop when the validation objective $J(S)$ has no further strict improvement (i.e., $\Delta J \le \epsilon_{\text{imp}}$). The score includes a length penalty $\lambda_{\text{len}}$.

\item \textbf{PNS+Greedy (pruned)}: same selection, followed by \(L_0\)-style pruning using the same $J(S)$; remove a rule if the decrease is at most $\tau_{\text{prune}}$.

\item \textbf{Corr+Greedy (Pearson)}: replace PNS with the Pearson correlation \(r\) (treatment vs.\ target-band attainment) for ranking; otherwise identical to the greedy procedure.

\item \textbf{Corr+Greedy (Pearson, pruned)}: correlation-based greedy selection with the same \(L_0\) pruning.

\item \textbf{Coverage-Driven}: rank by unconditional target-band coverage on treated episodes (ignoring baseline failure); favors breadth over causal guarantees.

\item \textbf{Coverage+Corr Hybrid}: rank by \(\alpha\,\overline{\mathrm{Cov}} + (1-\alpha)\,\overline{r}\), with the same greedy and pruning steps.

\item \textbf{Majority Router}: for each \((\theta_0,\mathrm{norm})\), choose the single task that wins the most paired episodes on \texttt{train}; output one unconditional rule (targets=\texttt{ALL}).

\item \textbf{Random Router}: uniformly random permutation of candidate tasks with a fixed RNG seed \(s\) (default \(s\) in App.~\ref{app:hparams}); a weak, non-causal baseline.
\end{itemize}


\paragraph{Metrics.}
\textbf{PNS\(_\text{Target}\)} (Train/Test): measures \emph{causal} band entry for the target group(s) after switching policy—originally out of band, now in band.
\textbf{Coverage} (Train/Test): plain hit rate under the current distribution—what fraction are in band \emph{now} under $\theta$; ignores where they would be under \( \theta_0\). No counterfactuals.
\textbf{Rules}\,$\downarrow$: number of clauses in the learned first-match router (smaller is better). 
\textbf{Gen.\ Gap}\,$\downarrow$: train--test drop in causal effectiveness,
\begin{equation}
\mathrm{Gap} \;=\; \mathrm{PNS}_{\text{Target}}^{\text{train}} \;-\; \mathrm{PNS}_{\text{Target}}^{\text{test}}.
\end{equation}

\textbf{Perf.}\,$\uparrow$: overall test score that rewards causal effectiveness and basic attainability while penalizing overfitting and complexity,
\begin{equation}
\mathrm{Perf}
\;=\;
w_{\mathrm{pns}}\,\mathrm{PNS}_{\text{Target}}^{\text{test}}
\;+\;
w_{\mathrm{cov}}\,\mathrm{Coverage}^{\text{test}}
\;-\;
w_{\mathrm{gap}}\,\mathrm{Gap}
\;-\;
\lambda_{\text{len}} \cdot \mathrm{Rules}_{\mathrm{norm}}.
\end{equation}
where $\mathrm{Rules}_{\mathrm{norm}}=\mathrm{Rules}/K$, and $K$ is a normalization constant
used only for the length penalty (set to the maximum router length per table; $K{=}80$ in our ablations).

\subsection{Experimental Results}

\subsubsection{Experiment 1: Existence and Rule Discovery of Invariant Causal Routing}
\label{exp:existence}

\begin{table}[t]
\centering
\caption{Representative PNS-significant causal routes selected by Stage~I and Stage~II. Entries report PNS with 95\% binomial confidence intervals as \emph{value [lower, upper]}. }
\label{tab:exp1_pns_ic}
\begin{tabular}{l|ccc|c}
\toprule
Norm & Baseline & Current & Group & PNS \\
\midrule
RI-1 & NI         & FAI       & ALL   & 0.966 [0.904, 0.993] \\
RI-1 & GMV        & FAI       & ALL   & 0.981 [0.899, 1.000] \\
RI-1 & BAL        & FAI       & ALL   & 0.962 [0.893, 0.992] \\
ST-1 & BAL        & GMV       & ALL   & 0.857 [0.722, 0.933] \\
ST-1 & FAI        & GMV       & ALL   & 0.887 [0.774, 0.947] \\
ST-1 & UW         & GMV       & ALL   & 0.905 [0.779, 0.962] \\
RI-2 & NI         & BAL       & ALL   & 0.833 [0.735, 0.900] \\
RI-2 & FAI        & BAL       &  ALL   & 0.864 [0.761, 0.927] \\
\bottomrule
\end{tabular}
\end{table}

\begin{table*}[t]
\centering
\caption{Ablation results on disjoint train/test seeds (Perf recomputed with $K{=}80$). 
Higher $\uparrow$ is better; lower $\downarrow$ is better.}
\label{tab:exp2_ablation_style_changed}
\small
\setlength{\tabcolsep}{6pt}
\begin{tabular}{l|cc|cc|cc|ccc}
\toprule
\textbf{Method} & PNS$_\text{train}$ $\uparrow$ & Cov$_\text{train}$ $\uparrow$ & PNS$_\text{test}$ $\uparrow$ & Cov$_\text{test}$ $\uparrow$ & Rules $\downarrow$ & Rules$_\text{norm}$ $\downarrow$ & Gap $\downarrow$ & Perf. $\uparrow$ \\
\midrule
\rowcolor[HTML]{F2F2F2}
PNS+Greedy & \textbf{0.989} & \textbf{0.990} & \textbf{0.953} & \textbf{0.966} & 24 & 0.300 & \underline{0.036} & \underline{0.862} \\
\rowcolor[HTML]{F2F2F2}
PNS+Greedy (pruned) & \underline{0.972} & \underline{0.981} & \underline{0.931} & \underline{0.938} & \textbf{12} & \textbf{0.150} & 0.041 & \textbf{0.883} \\
Corr+Greedy (Pearson) & 0.805 & 0.958 & 0.741 & 0.931 & 46 & 0.575 & 0.064 & 0.600 \\
Corr+Greedy (Pearson, pruned) & 0.742 & 0.840 & 0.677 & 0.796 & \underline{18} & \underline{0.225} & 0.065 & 0.627 \\
Coverage-Driven & 0.622 & 0.944 & 0.565 & 0.915 & 48 & 0.600 & 0.057 & 0.449 \\
Coverage+Corr Hybrid & 0.416 & 0.564 & 0.384 & 0.550 & 80 & 1.000 & \textbf{0.032} & 0.114 \\
Majority Router & 0.396 & 0.540 & 0.353 & 0.519 & 20 & 0.250 & 0.043 & 0.307 \\
Random Router & 0.294 & 0.393 & 0.251 & 0.354 & 60 & 0.750 & 0.043 & 0.042 \\
\bottomrule
\end{tabular}
\end{table*}

\begin{figure}[ht]
  \centering
\includegraphics[width=0.5\textwidth]{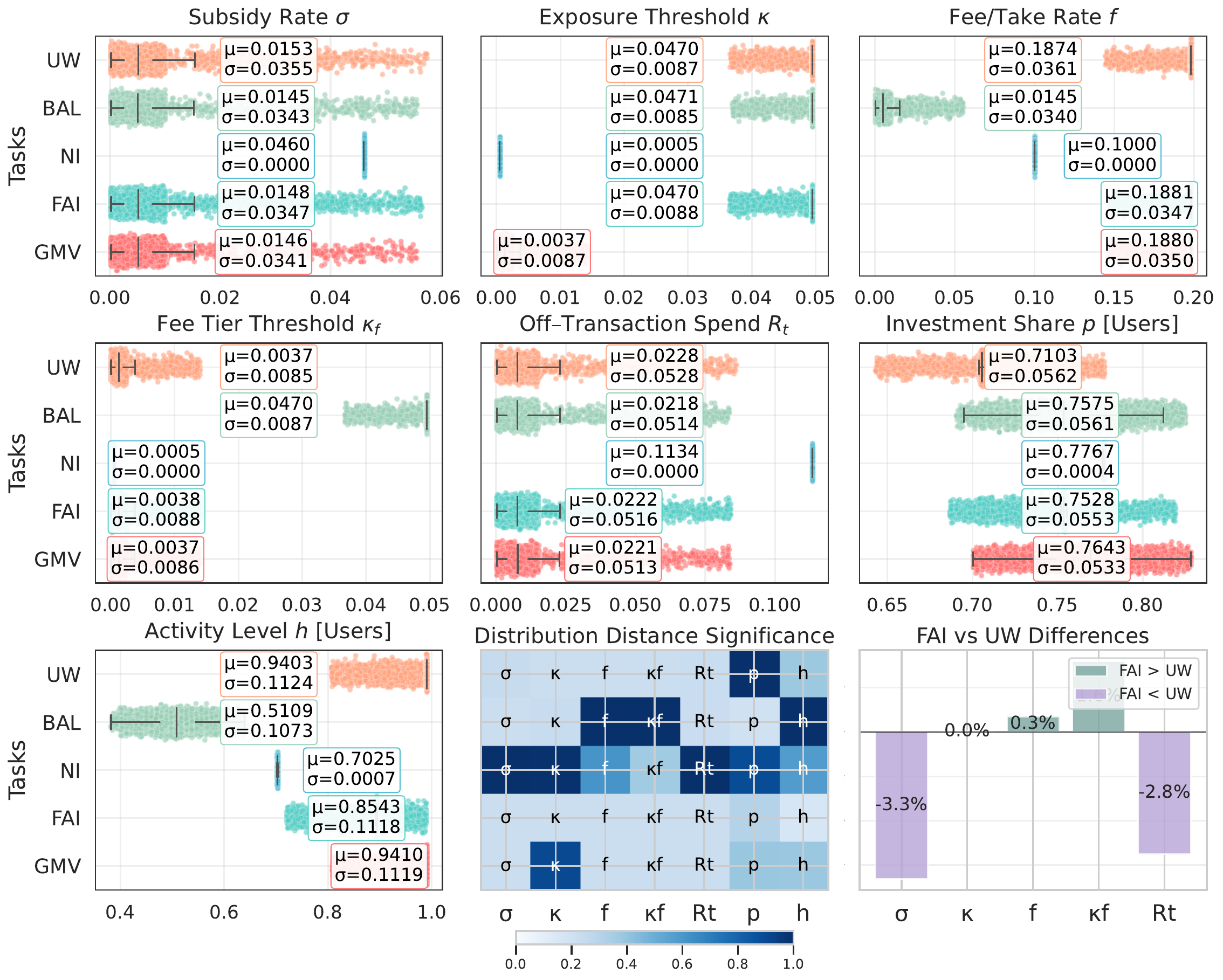}
  \caption{Key factors attribution for invariant causal routes.}
  \label{fig:exp3_strip}
\end{figure}

\paragraph{Visual evidence of norm existence.}  
Using one training-set seed, we record the final 50 steps after the system has fully evolved.  
Fig.~\ref{fig:exp1_timeseries} plots these trajectories across four norms (rows) and five tasks (columns).  
Group trajectories (low/mid/high-resource) relative to shaded bands confirm that multiple norms are stably attained under specific tasks:  
(i) \textbf{ST-1} is sustained under GMV;  
(ii) \textbf{ST-2} emerges under BAL;  
(iii) \textbf{RI-1} appears clearly under FAI; and  
(iv) \textbf{RI-2} requires BAL, while NI fails to hold bands.  
\textbf{These results provide direct evidence that social norms manifest within a simulated online market.}  

\paragraph{Invariant causal routes.}

For example, \emph{RI-1: NI$\to$FAI (ALL)} with \(\mathrm{PNS}=0.966\) means that replacing normal intervention (NI) by a fairness-oriented strategy (FAI) \emph{causes} all groups to attain the RI-1 band with probability \(0.966\) under the covered contexts. Likewise, \emph{ST-1: BAL$\to$GMV (ALL)} shows that switching from balanced growth–fairness (BAL) to a GMV-focused objective \emph{causes} the subsidy-to-transaction ratio \((\phi_1)\) to enter the ST-1 band.

PNS-certified clauses provide \textbf{causal validity} (RQ1): they quantify how strategies differentially induce norm attainment across groups.
The Stage-II router provides \textbf{actionable parsimony} (RQ2): a short, feasible rule list that generalizes across seeds/contexts while maintaining high compliance.

\paragraph{Toward generalization.}  
A key strength of this design is that PNS-based selection eliminates confounding and certifies invariance through twin-world estimation.  
As a result, the identified routes generalize across seeds without re-tuning—unlike correlation-based approaches that drift with confounder shifts.  
The next section (Exp.~\ref{exp:ablation}) evaluates this generalization explicitly, showing that PNS-based routers maintain strong compliance and small gaps on disjoint test seeds.

\subsubsection{Experiment 2: Ablations of Invariant Causal Routing}
\label{exp:ablation}



\paragraph{Results and analysis.}
\textbf{PNS+Greedy} and \textbf{PNS+Greedy (pruned)} dominate on test-time causal effectiveness and generalization under seed shift, reflecting the benefit of selecting clauses with certified necessity–sufficiency. Pruning improves parsimony with negligible loss—yielding shorter, cleaner rule lists while preserving causal validity. \textbf{Corr+Greedy (Pearson)} and \textbf{Corr+Greedy (Pearson, pruned)}—are weaker under shift, showing larger gaps and lower causal attainment, consistent with sensitivity to spurious associations. \textbf{Coverage-Driven} and \textbf{Coverage+Corr Hybrid} trade causality for breadth: they can raise unconditional hit rate but require long lists and lack counterfactual accountability, which limits robustness. \textbf{Majority Router} and \textbf{Random Router} serve as lower bounds on both effectiveness. 

Overall, our PNS-guided router \emph{realizes invariant causal routing} under distribution shift: it preserves certified causal effects across disjoint seeds while achieving strong test-time performance with compact rule lists. 

\subsubsection{Experiment 3: Interpretability and Plausibility of Invariant Causal Routing}
\label{exp:mech}

\paragraph{Platform Levers and User Responses} By stage~III, we compare the \emph{endogenous} distributions of (i) \textbf{platform levers}: subsidy rate $(\sigma)$, exposure threshold $(\kappa)$, commission $(f)$, fee–tier threshold $(\kappa_f)$\footnote{Minimum volume to qualify for a lower take-rate tier; distinct from the take-rate $f$.}, and off–transaction spend $(R_t)$\footnote{$R_t \equiv \absspend = \platformspend \cdot Y_t$, where $\platformspend$ is the spend share and $Y_t$ is aggregate GMV. Off–transaction spend covers marketing/ops, trust \& safety, and tooling.}, and (ii) \textbf{user responses}: investment share $(p)$ and activity level $(h)$. 
Each task fixes a policy mapping $\pi^{\text{task}}$. The platform applies $\pi^{\text{task}}$, users react with $(p,h)$, and the system converges to a \emph{social norm}. 
Fig.~\ref{fig:exp3_strip} shows the per–task distributions for levers and responses, highlights levers with significant cross–task differences, and provides the \emph{FAI} vs.\ \emph{UW} mean deltas (bottom–right). 

\paragraph{Why switching tasks produces these PNS routes.}For example:

\textbf{BAL $\rightarrow$ GMV $\Rightarrow$ ST-1 (Route 4 in Table~\ref{tab:exp1_pns_ic}).}  
Fig.~\ref{fig:exp3_strip} further shows that $\kappa$ (exposure threshold) and $\kappa_f$ (fee–tier threshold) decline while the commission $f$ (platform take rate) increases.  
Lower thresholds make it easier to gain exposure and to qualify for lower fee tiers, which drives $h$ (activity) upward. 
Volume rises but subsidies $\sigma$ remain flat, so per-transaction subsidy intensity is low across groups, again producing \emph{ST-1} .

\textbf{NI $\rightarrow$ FAI $\Rightarrow$ RI-1 (Route 1 in Table~\ref{tab:exp1_pns_ic}).}
Under \emph{FAI}, the platform reduces off–transaction spend $R_t$ and lowers the fee–tier threshold $\kappa_f$.
The lower $\kappa_f$ directly expands access to low-fee tiers, allowing more low-resource users to qualify for a lower \emph{realized} take rate $f$ and improving their relative inclusion (RI); meanwhile, high-resource users’ relative advantage narrows. 
This pattern is consistent with \emph{RI-1}.


\paragraph{PNS separates even subtle cases (FAI vs.\ UW)}  
In the lever panels, FAI and UW almost overlap. The bottom–right plot shows mean gaps of only a few percent: FAI has slightly higher $\kappa_f$ (fee–tier threshold) and $f$ (commission), while UW has slightly higher $R_t$ (off–transaction spend) and $\sigma$ (subsidy rate).  
Even such small, systematic shifts alter behavior at the margin: under FAI, the activity level $h$ is slightly lower but the investment share $p$ is higher, whereas under UW the reverse holds. 
PNS distinguishes these cases by tracing the full pathway—\emph{fixed policy} $\rightarrow$ \emph{endogenous levers} $(\sigma,\kappa,f,\kappa_f,R_t)$ $\rightarrow$ \emph{user responses} $(p,h)$ $\rightarrow$ \emph{social norms}—even when visual differences are minimal.

These results thus answer RQ3 by revealing interpretable, causally plausible routes of norm formation.



\section{Conclusion}

This work introduced \textbf{Invariant Causal Routing (ICR)}, a causal governance framework for discovering stable policy–norm relations under distribution shift. ICR integrates causal identification, compact rule learning, and key factor attribution to isolate genuine causal effects and construct interpretable governance strategies. Through heterogeneous agent simulations calibrated with real data, ICR demonstrates improved norm stability, smaller generalization gaps, and more concise rules compared with correlation or coverage based baselines. These results highlight the value of causal invariance as a foundation for reliable and interpretable governance in complex systems.

Beyond this study, the causal routing perspective opens several promising directions. Future work will extend ICR to multi-level governance with adaptive agents, explore causal meta-learning for dynamic environments, and validate the framework using empirical data from large-scale digital platforms. We hope this work contributes to building principled, transparent, and transferable approaches for governing social norms in evolving socio-economic ecosystems.

\newpage
\bibliographystyle{ACM-Reference-Format}
\bibliography{sample-base}

@inproceedings{yang2025twinmarket,
  author    = {Yuzhe Yang and Yifei Zhang and Minghao Wu and Kaidi Zhang and Yunmiao Zhang and Honghai Yu and Yan Hu and Benyou Wang},
  title     = {TwinMarket: A Scalable Behavioral and Social Simulation for Financial Markets},
  booktitle = {ICLR 2025 Workshop on Advances in Financial AI},
  year      = {2025},
  note      = {arXiv:2502.01506}
}

@book{meyn2009markov,
  title={Markov Chains and Stochastic Stability},
  author={Meyn, Sean P. and Tweedie, Richard L.},
  year={2009},
  publisher={Springer}
}

@book{kallenberg2002foundations,
  title={Foundations of Modern Probability},
  author={Kallenberg, Olaf},
  year={2002},
  publisher={Springer}
}

@book{billingsley1999convergence,
  title={Convergence of Probability Measures},
  author={Billingsley, Patrick},
  year={1999},
  publisher={Wiley}
}

@book{pearl2009causality,
  title={Causality},
  author={Pearl, Judea},
  year={2009},
  publisher={Cambridge university press}
}

@book{bicchieri2005grammar,
  title={The grammar of society: The nature and dynamics of social norms},
  author={Bicchieri, Cristina},
  year={2005},
  publisher={Cambridge University Press}
}

@article{young1993evolution,
  title={The evolution of conventions},
  author={Young, H Peyton},
  journal={Econometrica: Journal of the Econometric Society},
  pages={57--84},
  year={1993},
  publisher={JSTOR}
}

@article{young2015evolution,
  title={The evolution of social norms},
  author={Young, H Peyton},
  journal={Annual Review of Economics},
  volume={7},
  number={1},
  pages={359--387},
  year={2015},
  publisher={Annual Reviews}
}

@article{shoham1995social,
  title={On social laws for artificial agent societies: off-line design},
  author={Shoham, Yoav and Tennenholtz, Moshe},
  journal={Artificial intelligence},
  volume={73},
  number={1-2},
  pages={231--252},
  year={1995},
  publisher={Elsevier}
}

@article{shoham1997emergence,
  title={On the emergence of social conventions: modeling, analysis, and simulations},
  author={Shoham, Yoav and Tennenholtz, Moshe},
  journal={Artificial Intelligence},
  volume={94},
  number={1-2},
  pages={139--166},
  year={1997},
  publisher={Elsevier}
}

@article{krupka2013identifying,
  title={Identifying social norms using coordination games: Why does dictator game sharing vary?},
  author={Krupka, Erin L and Weber, Roberto A},
  journal={Journal of the European Economic Association},
  volume={11},
  number={3},
  pages={495--524},
  year={2013},
  publisher={Oxford University Press}
}

@inproceedings{sen2007emergence,
  title={Emergence of norms through social learning.},
  author={Sen, Sandip and Airiau, St{\'e}phane},
  booktitle={IJCAI},
  volume={1507},
  pages={1512},
  year={2007}
}

@article{hawkins2019emergence,
  title={The emergence of social norms and conventions},
  author={Hawkins, Robert XD and Goodman, Noah D and Goldstone, Robert L},
  journal={Trends in cognitive sciences},
  volume={23},
  number={2},
  pages={158--169},
  year={2019},
  publisher={Elsevier}
}

@article{opp1982evolutionary,
  title={The evolutionary emergence of norms},
  author={Opp, Karl-Dieter},
  journal={British journal of social psychology},
  volume={21},
  number={2},
  pages={139--149},
  year={1982},
  publisher={Wiley Online Library}
}

@article{chung2016social,
  title={Social norms: A review},
  author={Chung, Adrienne Chung Adrienne and Rimal, Rajiv N Rimal Rajiv N},
  journal={Review of Communication Research},
  volume={4},
  pages={01--28},
  year={2016}
}

@article{shaffer1983toward,
  title={Toward Pepitone's vision of a normative social psychology: What is a social norm?},
  author={Shaffer, Leigh S},
  journal={The journal of mind and behavior},
  pages={275--293},
  year={1983},
  publisher={JSTOR}
}

@article{berkowitz2005overview,
  title={An overview of the social norms approach},
  author={Berkowitz, Alan D},
  journal={Changing the culture of college drinking: A socially situated health communication campaign},
  volume={1},
  pages={193--214},
  year={2005}
}

@article{hechter2001social,
  title={Social norms},
  author={Hechter, Michael and Opp, Karl-Dieter},
  year={2001},
  publisher={Russell Sage Foundation}
}

@article{dai2024artificial,
  title={Artificial leviathan: Exploring social evolution of llm agents through the lens of hobbesian social contract theory},
  author={Dai, Gordon and Zhang, Weijia and Li, Jinhan and Yang, Siqi and Rao, Srihas and Caetano, Arthur and Sra, Misha and others},
  journal={arXiv preprint arXiv:2406.14373},
  year={2024}
}

@book{angrist2009mostly,
  title={Mostly harmless econometrics: An empiricist's companion},
  author={Angrist, Joshua D and Pischke, J{\"o}rn-Steffen},
  year={2009},
  publisher={Princeton university press}
}

@article{rosenbaum1983central,
  title={The central role of the propensity score in observational studies for causal effects},
  author={Rosenbaum, Paul R and Rubin, Donald B},
  journal={Biometrika},
  volume={70},
  number={1},
  pages={41--55},
  year={1983},
  publisher={Oxford University Press}
}

@article{dudik2011doubly,
  title={Doubly robust policy evaluation and learning},
  author={Dud{\'\i}k, Miroslav and Langford, John and Li, Lihong},
  journal={arXiv preprint arXiv:1103.4601},
  year={2011}
}

@article{abadie2010synthetic,
  title={Synthetic control methods for comparative case studies: Estimating the effect of California’s tobacco control program},
  author={Abadie, Alberto and Diamond, Alexis and Hainmueller, Jens},
  journal={Journal of the American statistical Association},
  volume={105},
  number={490},
  pages={493--505},
  year={2010},
  publisher={Taylor \& Francis}
}

@book{imbens2015causal,
  title={Causal inference in statistics, social, and biomedical sciences},
  author={Imbens, Guido W and Rubin, Donald B},
  year={2015},
  publisher={Cambridge university press}
}

@article{athey2017state,
  title={The state of applied econometrics: Causality and policy evaluation},
  author={Athey, Susan and Imbens, Guido W},
  journal={Journal of Economic perspectives},
  volume={31},
  number={2},
  pages={3--32},
  year={2017},
  publisher={American Economic Association 2014 Broadway, Suite 305, Nashville, TN 37203-2418}
}

@article{wager2018estimation,
  title={Estimation and inference of heterogeneous treatment effects using random forests},
  author={Wager, Stefan and Athey, Susan},
  journal={Journal of the American Statistical Association},
  volume={113},
  number={523},
  pages={1228--1242},
  year={2018},
  publisher={Taylor \& Francis}
}

@article{kunzel2019metalearners,
  title={Metalearners for estimating heterogeneous treatment effects using machine learning},
  author={K{\"u}nzel, S{\"o}ren R and Sekhon, Jasjeet S and Bickel, Peter J and Yu, Bin},
  journal={Proceedings of the national academy of sciences},
  volume={116},
  number={10},
  pages={4156--4165},
  year={2019},
  publisher={National Academy of Sciences}
}

@article{bareinboim2016causal,
  title={Causal inference and the data-fusion problem},
  author={Bareinboim, Elias and Pearl, Judea},
  journal={Proceedings of the National Academy of Sciences},
  volume={113},
  number={27},
  pages={7345--7352},
  year={2016},
  publisher={National Academy of Sciences}
}

@article{peters2016causal,
  title={Causal inference by using invariant prediction: identification and confidence intervals},
  author={Peters, Jonas and B{\"u}hlmann, Peter and Meinshausen, Nicolai},
  journal={Journal of the Royal Statistical Society Series B: Statistical Methodology},
  volume={78},
  number={5},
  pages={947--1012},
  year={2016},
  publisher={Oxford University Press}
}

@misc{hernan2010causal,
  title={Causal inference},
  author={Hern{\'a}n, Miguel A and Robins, James M},
  year={2010},
  publisher={CRC Boca Raton, FL}
}

@article{athey2021policy,
  title={Policy learning with observational data},
  author={Athey, Susan and Wager, Stefan},
  journal={Econometrica},
  volume={89},
  number={1},
  pages={133--161},
  year={2021},
  publisher={Wiley Online Library}
}

@inproceedings{jiang2016doubly,
  title={Doubly robust off-policy value evaluation for reinforcement learning},
  author={Jiang, Nan and Li, Lihong},
  booktitle={International conference on machine learning},
  pages={652--661},
  year={2016},
  organization={PMLR}
}

@inproceedings{thomas2016data,
  title={Data-efficient off-policy policy evaluation for reinforcement learning},
  author={Thomas, Philip and Brunskill, Emma},
  booktitle={International conference on machine learning},
  pages={2139--2148},
  year={2016},
  organization={PMLR}
}

@article{dulac2019challenges,
  title={Challenges of real-world reinforcement learning},
  author={Dulac-Arnold, Gabriel and Mankowitz, Daniel and Hester, Todd},
  journal={arXiv preprint arXiv:1904.12901},
  year={2019}
}

@inproceedings{swaminathan2015counterfactual,
  title={Counterfactual risk minimization: Learning from logged bandit feedback},
  author={Swaminathan, Adith and Joachims, Thorsten},
  booktitle={International conference on machine learning},
  pages={814--823},
  year={2015},
  organization={PMLR}
}

@inproceedings{schnabel2016recommendations,
  title={Recommendations as treatments: Debiasing learning and evaluation},
  author={Schnabel, Tobias and Swaminathan, Adith and Singh, Ashudeep and Chandak, Navin and Joachims, Thorsten},
  booktitle={international conference on machine learning},
  pages={1670--1679},
  year={2016},
  organization={PMLR}
}

@article{rudin2019stop,
  title={Stop explaining black box machine learning models for high stakes decisions and use interpretable models instead},
  author={Rudin, Cynthia},
  journal={Nature machine intelligence},
  volume={1},
  number={5},
  pages={206--215},
  year={2019},
  publisher={Nature Publishing Group UK London}
}

@inproceedings{joachims2017unbiased,
  title={Unbiased learning-to-rank with biased feedback},
  author={Joachims, Thorsten and Swaminathan, Adith and Schnabel, Tobias},
  booktitle={Proceedings of the tenth ACM international conference on web search and data mining},
  pages={781--789},
  year={2017}
}

@article{dellavigna2018motivates,
  title={What motivates effort? Evidence and expert forecasts},
  author={DellaVigna, Stefano and Pope, Devin},
  journal={The Review of Economic Studies},
  volume={85},
  number={2},
  pages={1029--1069},
  year={2018},
  publisher={Oxford University Press}
}

@article{fehr2000cooperation,
  title={Cooperation and punishment in public goods experiments},
  author={Fehr, Ernst and G{\"a}chter, Simon},
  journal={American Economic Review},
  volume={90},
  number={4},
  pages={980--994},
  year={2000},
  publisher={American Economic Association}
}

@book{lewis2008convention,
  title={Convention: A philosophical study},
  author={Lewis, David},
  year={2008},
  publisher={John Wiley \& Sons}
}

@misc{scf2022,
  author       = {{Board of Governors of the Federal Reserve System}},
  title        = {2022 Survey of Consumer Finances (SCF)},
  year         = {2023},
  publisher    = {Board of Governors of the Federal Reserve System},
  howpublished = {\url{https://www.federalreserve.gov/econres/scfindex.htm}},
  doi          = {10.17016/8799},
  note         = {The Survey of Consumer Finances (SCF) is a triennial cross-sectional survey of U.S. families, providing data on balance sheets, pensions, income, and demographic characteristics.}
}

@article{tian2000probabilities,
  title={Probabilities of causation: Bounds and identification},
  author={Tian, Jin and Pearl, Judea},
  journal={Annals of Mathematics and Artificial Intelligence},
  volume={28},
  number={1},
  pages={287--313},
  year={2000},
  publisher={Springer}
}

@article{pfister2019invariant,
  title={Invariant causal prediction for sequential data},
  author={Pfister, Niklas and B{\"u}hlmann, Peter and Peters, Jonas},
  journal={Journal of the American Statistical Association},
  volume={114},
  number={527},
  pages={1264--1276},
  year={2019},
  publisher={Taylor \& Francis}
}

@article{rothenhausler2021anchor,
  title={Anchor regression: Heterogeneous data meet causality},
  author={Rothenh{\"a}usler, Dominik and Meinshausen, Nicolai and B{\"u}hlmann, Peter and Peters, Jonas},
  journal={Journal of the Royal Statistical Society Series B: Statistical Methodology},
  volume={83},
  number={2},
  pages={215--246},
  year={2021},
  publisher={Oxford University Press}
}

@article{heinze2021conditional,
  title={Conditional variance penalties and domain shift robustness},
  author={Heinze-Deml, Christina and Meinshausen, Nicolai},
  journal={Machine Learning},
  volume={110},
  number={2},
  pages={303--348},
  year={2021},
  publisher={Springer}
}

@article{gamella2020active,
  title={Active invariant causal prediction: Experiment selection through stability},
  author={Gamella, Juan L and Heinze-Deml, Christina},
  journal={Advances in Neural Information Processing Systems},
  volume={33},
  pages={15464--15475},
  year={2020}
}

@inproceedings{oberst2021regularizing,
  title={Regularizing towards causal invariance: Linear models with proxies},
  author={Oberst, Michael and Thams, Nikolaj and Peters, Jonas and Sontag, David},
  booktitle={International Conference on Machine Learning},
  pages={8260--8270},
  year={2021},
  organization={PMLR}
}

@inproceedings{zhang2015multi,
  title={Multi-source domain adaptation: A causal view},
  author={Zhang, Kun and Gong, Mingming and Sch{\"o}lkopf, Bernhard},
  booktitle={Proceedings of the AAAI Conference on Artificial Intelligence},
  volume={29},
  number={1},
  year={2015}
}

@inproceedings{ahuja2020invariant,
  title={Invariant risk minimization games},
  author={Ahuja, Kartik and Shanmugam, Karthikeyan and Varshney, Kush and Dhurandhar, Amit},
  booktitle={International Conference on Machine Learning},
  pages={145--155},
  year={2020},
  organization={PMLR}
}

@inproceedings{lu2021invariant,
  title={Invariant causal representation learning for out-of-distribution generalization},
  author={Lu, Chaochao and Wu, Yuhuai and Hern{\'a}ndez-Lobato, Jos{\'e} Miguel and Sch{\"o}lkopf, Bernhard},
  booktitle={International Conference on Learning Representations},
  year={2021}
}

@article{gulrajani2020search,
  title={In search of lost domain generalization},
  author={Gulrajani, Ishaan and Lopez-Paz, David},
  journal={arXiv preprint arXiv:2007.01434},
  year={2020}
}

@article{yang2023invariant,
  title={Invariant learning via probability of sufficient and necessary causes},
  author={Yang, Mengyue and Fang, Zhen and Zhang, Yonggang and Du, Yali and Liu, Furui and Ton, Jean-Francois and Wang, Jianhong and Wang, Jun},
  journal={Advances in Neural Information Processing Systems},
  volume={36},
  pages={79832--79857},
  year={2023}
}

\newpage
\appendix

\section{Assumptions and RL Plausibility}
\label{app:assump}

\paragraph{Compactness and Feller Continuity: Sufficient Primitives.}
Assume that the market evolution is governed by a continuous transition function $F^{(\theta)}$ with i.i.d. exogenous noise $\xi_t$. Specifically, if $X_{t+1} = F^{(\theta)}(X_t, \xi_t)$, and $F^{(\theta)}$ is continuous while $\xi_t$ is i.i.d. with fixed law, then for any bounded continuous function $f \in C_b(\mathcal X)$, the transition kernel $P^{(\theta)}$ defined by $P^{(\theta)}f(x) = \mathbb{E}[f(F^{(\theta)}(x, \xi))]$ is continuous by the dominated convergence theorem. Consequently, $P^{(\theta)}$ is Feller. Compactness of the state space can be enforced via clipping and bounded encoders.

\paragraph{Alternative Assumptions (used only if necessary).}
We could alternatively assume \textbf{weak-Feller + tightness} on Polish spaces, or \textbf{drift + minorization} (positive Harris recurrence). These provide drop-in replacements for existence and convergence theorems. However, the main text does not rely on these alternative sets of assumptions.

\section{Existence via Krylov--Bogolyubov}
\label{app:invariant}

Standing assumptions: $(\mathcal X, d)$ is a compact metric space, and $P$ is Feller on $(\mathcal X, \mathcal B)$.

\begin{lemma}[Tightness of Ces\`aro averages]
\label{lem:invariant}
Fix $x_0 \in \mathcal X$. Define the Ces\`aro averages of the Markov chain as $\nu_T \triangleq \frac{1}{T} \sum_{t=0}^{T-1} P^t(x_0, \cdot)$. Then the family $\{\nu_T\}_{T \ge 1}$ is tight in $\mathcal P(\mathcal X)$, where $\mathcal P(\mathcal X)$ denotes the space of probability measures on $\mathcal X$.
\end{lemma}

\begin{proof}
Since $\mathcal X$ is compact, every probability measure on $\mathcal X$ is tight. Moreover, $\mathcal P(\mathcal X)$ is compact in the weak topology, which guarantees the tightness of $\{\nu_T\}_{T \ge 1}$.
\end{proof}

\begin{lemma}[Invariance of weak limits]
\label{lem:invariant-weak-limit}
If $\nu_{T_k} \Rightarrow \pi$ weakly, then $\pi$ is invariant for $P$.
\end{lemma}

\begin{proof}
Let $f \in C_b(\mathcal X)$. We have
\[
\int (P f) \, d\nu_T - \int f \, d\nu_T = \frac{1}{T} \left(\int f \, dP^T(x_0, \cdot) - f(x_0) \right) \xrightarrow[T \to \infty]{} 0.
\]
Since $P$ is Feller, $P f \in C_b(\mathcal X)$. By weak convergence, 
\[
\int (P f) \, d\nu_{T_k} \to \int (P f) \, d\pi \quad \text{and} \quad \int f \, d\nu_{T_k} \to \int f \, d\pi.
\]
Therefore, $\int (P f) \, d\pi = \int f \, d\pi$ for all $f \in C_b(\mathcal X)$, implying that $\pi P = \pi$.
\end{proof}

\begin{theorem}[Krylov--Bogolyubov Theorem]
\label{thm:KB}
The set of invariant measures $\mathcal I(P)$ is nonempty, convex, and weakly compact. In particular, for any $\phi \in C_b(\mathcal X)$, the signature set
\[
\Sigma(P, \phi) = \left\{ \int \phi \, d\pi : \pi \in \mathcal I(P) \right\}
\]
is nonempty and compact.
\end{theorem}

\paragraph*{References.}
\citet{meyn2009markov, kallenberg2002foundations, billingsley1999convergence}.

\section{Occupation Measures and Time Averages}
\label{app:occ}

Standing assumptions: $(\mathcal X, d)$ is compact and $P$ is Feller. The process $\{X_t\}_{t \ge 0}$ is a Markov chain with transition kernel $P$. Define $\mu_T \triangleq \frac{1}{T} \sum_{t=1}^T \delta_{X_t}$ as the occupation measure.

\begin{lemma}[Tightness of occupation measures]
\label{lem:C-tight}
The family $\{\mu_T\}_{T \ge 1}$ is tight in $\mathcal P(\mathcal X)$.
\end{lemma}

\begin{lemma}[Ces\`aro average of martingale differences]
\label{lem:C-md}
For $f \in C_b(\mathcal X)$, define $M_{t+1} \triangleq f(X_{t+1}) - (P f)(X_t)$ as a bounded martingale difference. Then $\frac{1}{T} \sum_{t=0}^{T-1} M_{t+1} \to 0$ in $L^1$ and almost surely.
\end{lemma}

\begin{lemma}[Invariance of weak limits of $\mu_T$]
\label{lem:C-inv}
If $\mu_{T_k} \Rightarrow \mu_\star$ weakly, then $\mu_\star \in \mathcal I(P)$.
\end{lemma}

\begin{theorem}[Limit points of time averages lie in $\Sigma$]
\label{thm:C-sigma}
Let $\phi \in C_b(\mathcal X)$ and $\bar\phi_T = \int \phi \, d\mu_T$. Every subsequential limit of $(\bar\phi_T)$ equals $\int \phi \, d\pi$ for some $\pi \in \mathcal I(P)$. Hence, all limit points belong to $\Sigma(P, \phi)$.
\end{theorem}

\label{rem:C-borel}
If $\phi$ is only a bounded Borel function, weak convergence need not imply continuity of integrals. However, the conclusion of Theorem~\ref{thm:C-sigma} still holds if $\phi$ is $\pi$-almost surely continuous for every $\pi \in \mathcal I(P)$ or if there exist $(\phi_n) \subset C_b(\mathcal X)$ such that $\|\phi_n - \phi\|_{L^1(\pi)} \to 0$ uniformly over $\pi \in \mathcal I(P)$.

\section{Convergence of Long-Run Statistics}
\label{app:conv}
The main text does \emph{not} require convergence assumptions. For completeness we recall two sufficient conditions.

\begin{definition}[$\phi$-uniqueness]\label{def:D-phiuniq}
$P$ has \emph{$\phi$-uniqueness} if $\int\phi\,d\pi$ equals the same constant $c_\phi$ for all $\pi\in\mathcal I(P)$.
\end{definition}

\begin{theorem}[$\phi$-uniqueness $\Rightarrow$ convergence]\label{thm:D-phi}
Under $\phi$-uniqueness, $\bar\phi_T\to c_\phi$ almost surely.
\end{theorem}

\paragraph{Harris ergodic LLN.}
If the chain is $\psi$-irreducible, aperiodic, and positive Harris recurrent, there is a unique invariant probability $\pi$, and $\frac{1}{T}\sum_{t=1}^T \phi(X_t)\to \int \phi\,d\pi$ a.s. for all integrable $\phi$ \citep[Thm.~17.0.1]{meyn2009markov}.

\section{Detailed Estimation Method and Confidence Interval Calculation}
\label{app:CI}

We estimate the Probability of Necessity and Sufficiency (PNS) using the following unbiased estimator:
\begin{equation*}
\widehat{\mathrm{PNS}}_g(\theta \mid \psi)
= \frac{1}{N_\psi} \sum_{e:\,\psi(X_e)=1}
\mathbf{1}\!\left\{\, Y_{g,e}(\theta)=1,\; Y_{g,e}(\theta_0)=0 \right\}.
\end{equation*}
where $N_\psi$ is the number of seeds satisfying $\psi(X_e)=1$.

For confidence intervals of this estimator, we use the Wilson score interval for a binomial proportion. The confidence intervals can be computed as follows:
Let $N = N_\psi$, $k=\sum_{e:\,\psi(X_e)=1}\mathbf{1}\{Y_{g,e}(\theta)=1,\ Y_{g,e}(\theta_0)=0\}$, and $\hat p = k/N$.

\textbf{Wilson score interval:}
\begin{equation*}
\left(
\frac{\hat p+\frac{z_{\alpha/2}^2}{2N}-z_{\alpha/2}\sqrt{\frac{\hat p(1-\hat p)}{N}+\frac{z_{\alpha/2}^2}{4N^2}}}{1+\frac{z_{\alpha/2}^2}{N}},
\quad 
\frac{\hat p+\frac{z_{\alpha/2}^2}{2N}+z_{\alpha/2}\sqrt{\frac{\hat p(1-\hat p)}{N}+\frac{z_{\alpha/2}^2}{4N^2}}}{1+\frac{z_{\alpha/2}^2}{N}}
\right).
\end{equation*}


\section{Experiment Setup}
\label{app:setup}

\subsection{Social Norm Ranges}

We instantiate per–group bands for ST (Subsidy/Transaction) and RI (Revenue/Investment). Group names follow the main text: \textbf{low-resource}, \textbf{mid-resource}, \textbf{high-resource}. Intervals use standard open/closed notation; $\infty$ denotes no upper bound.

\paragraph{ST-1 (low subsidy intensity across groups).}
\[
\begin{aligned}
\textbf{low-resource}  &:~ (0,\;0.7], \\
\textbf{mid-resource}  &:~ (0,\;1.0], \\
\textbf{high-resource} &:~ (0,\;1.3].
\end{aligned}
\]

\paragraph{ST-2 (regressive subsidy: higher for low-resource, lower for high-resource).}
\[
\begin{aligned}
\textbf{low-resource}  &:~ [3.5,\;\infty), \\
\textbf{mid-resource}  &:~ [2.3,\;\infty), \\
\textbf{high-resource} &:~ (0,\;2.0].
\end{aligned}
\]

\paragraph{RI-1 (upward mobility: higher RI for low-resource, lower for high-resource).}
\[
\begin{aligned}
\textbf{low-resource}  &:~ [1.05,\;\infty), \\
\textbf{high-resource} &:~ (0,\;0.4].
\end{aligned}
\]

\paragraph{RI-2 (concentration: lower RI for low-resource, higher for high-resource).}
\[
\begin{aligned}
\textbf{low-resource}  &:~ (0,\;0.62], \\
\textbf{high-resource} &:~ (0.8,\;\infty).
\end{aligned}
\]

\noindent
These bands are used for the “norm attainment” event defined in the main text: a run attains a norm if each group’s long-run statistic eventually stays within its band.

\subsection{Initial Configurations}
\label{app:ic}
Key parameters for five economies (CRRA=risk aversion, IFE=effort elasticity, $e_p$=superstar prob., $e_q$=stability, $\rho_e$=persistence, $\sigma_e$=shock, $super_e$=multiplier), see Table~\ref{tab:ic}.
Experiments 1 and 3 use IC1, whereas the ablation study (Experiment 2) uses IC1--IC5.

\begin{table}[ht]
\centering
\caption{Initial-condition (IC) configurations used in all experiments.}
\label{tab:ic}
\small
\begin{tabular}{lccccccc}
\toprule
IC & CRRA & IFE & $e_p$ & $e_q$ & $\rho_e$ & $\sigma_e$ & $super_e$ \\
\midrule
IC1 & 1.0 & 2.0 & $2.2\!\times\!10^{-6}$ & 0.990 & 0.982 & 0.20 & 504.3 \\
IC2 & 1.5 & 1.8 & $5.0\!\times\!10^{-6}$ & 0.985 & 0.975 & 0.18 & 400 \\
IC3 & 0.7 & 2.5 & $8.0\!\times\!10^{-6}$ & 0.985 & 0.990 & 0.15 & 350 \\
IC4 & 1.2 & 2.2 & $1.0\!\times\!10^{-5}$ & 0.920 & 0.950 & 0.22 & 450 \\
IC5 & 2.0 & 1.5 & $3.0\!\times\!10^{-6}$ & 0.995 & 0.990 & 0.25 & 600 \\
\bottomrule
\end{tabular}
\end{table}

\subsection{Action Space: Platform and Users}
\label{app:action-platform}

To keep optimization well-posed under budget/resource constraints, we parameterize controls by \emph{proportions} (ratios) rather than raw levels. This keeps the feasible set compact and stabilizes MARL training.

\paragraph{User actions.}
Each user \(i\) chooses an \emph{investment share} \(\userinvest_t^i \in (0,1)\) (e.g., ads/promo budget ratio) and an \emph{activity level} \(\useractivity_t^i \in [0, h_{\max}]\) (e.g., listing/engagement frequency). Let \(rev_t^i\) denote realized revenue and \(inv_t^i\) the investment stock for user \(i\). The platform applies a subsidy schedule \(S(\cdot; \subsidyrate_t, \exposureth_t)\) and a fee schedule \(F(\cdot; \feerate_t, {\feeth}_{t})\).\footnote{For example, tiered piecewise-linear schedules controlled by platform parameters; details omitted.}
Under the intertemporal budget,
\begin{equation*}
\userinvest_t^i
= \frac{inv_{t+1}^i}{\,rev_t^i + S\!\big(rev_t^i;\subsidyrate_t,\exposureth_t\big)
- F\!\big(rev_t^i;\feerate_t,{\feeth}_{t}\big) + inv_t^i\,},
\qquad
\useractivity_t^i \in [0,h_{\max}],
\end{equation*}
so decision-making over \((\text{spend}, \text{effort})\) is implemented via proportional controls \((\userinvest_t^i,\useractivity_t^i)\).

\paragraph{Platform actions.}
The platform sets high-level levers that shape incentives and aggregate allocation:
\begin{equation*}
\mathcal{A}^{\mathrm{P}}_t
= \{\, \subsidyrate_t,\ \exposureth_t,\ \feerate_t,\ {\feeth}_{t},\ \platformspend \,\},
\qquad
\absspend = \platformspend\, Y_t,
\end{equation*}
where $\subsidyrate_t$ and $\exposureth_t$ parameterize subsidies/exposure, $\feerate_t$ and ${\feeth}_{t}$ parameterize fees/take-rate tiers, $\absspend$ is the platform’s off–transaction spend, and $Y_t$ is an aggregate proxy (e.g., total GMV).

\paragraph{Action spaces.}
With proportional parameterization and clipping,
\begin{equation*}
\mathcal{A}^{u,i}_t
= \{\, \userinvest_t^i\in(0,1),\ \useractivity_t^i\in[0,h_{\max}] \,\},
\qquad
\mathcal{A}^{\mathrm{P}}_t \text{ as above}.
\end{equation*}
All actions are bounded to ensure compactness and continuity of the induced transition kernel.

\section{Training Parameters}
\label{app:train}

Table~\ref{tab:train_params} summarizes all training and optimization settings used in our experiments, including shared defaults and the ranges used in ablations. Unless otherwise noted, results use the Table~\ref{tab:train_params} defaults.

\begin{table}[t]
\centering
\caption{Training parameters. One step = one year; evaluation window $T{=}50$ (post burn-in).}
\label{tab:train_params}
\begin{tabular}{p{0.27\linewidth} p{0.65\linewidth}}
\hline
\textbf{Item} & \textbf{Value} \\
\hline
Population & $n_{\text{households}}=100$ \\
Episode length (env) & \textbf{300} steps (episode termination) \\
Model / Buffer & hidden size $=256$;\; replay buffer $=10^6$ \\
Optimization & Adam;\; $lr=3\times10^{-4}$;\; batch $=128$;\; $\gamma_{\text{RL}}=0.975$;\; $\tau=5\times10^{-3}$ \\
Schedule & $1500$ epochs;\; \emph{epoch length} $=500$ env steps (across episodes);\; update freq.=2;\; eval $=100$ episodes \\
PPO (aux) & $\gamma_{\text{PPO}}=0.99$;\; $\tau=0.95$;\; clip $=0.1$;\; $v_{\text{loss}}=0.5$;\; ent$=0.01$ \\
Exploration & warm-up $=1000$ steps;\; noise $=0.1$;\; $\epsilon=0.1$ \\
Randomization & train seeds $\{0,\ldots,7\}$;\; val $\{8,9\}$;\; test $\{10,\ldots,12\}$;\; twin-world PNS shares exogenous randomness for $(\theta,\theta_0)$ per seed \\
Checkpointing & save model every $100$ epochs \\
\hline
\end{tabular}
\end{table}

\subsection{Seed Configurations}
\label{app:ablation-seeds}

We use five independent contexts (IC1--IC5). Train and test seeds are strictly disjoint; at load time we drop any accidental overlaps.\footnote{Our runner removes from \texttt{test} any IC directory that also appears in \texttt{train}.} 
Table~\ref{tab:app-seeds} lists the exact seeds shipped with our logs; if a replication uses a different pool, the odd/even rule preserves disjointness.

\begin{table}[h]
\centering
\small
\caption{Seeds for ablation experiments. Train uses odd seeds; Test uses even seeds within each IC.}
\label{tab:app-seeds}
\begin{tabular}{lll}
\toprule
IC  & \textbf{Train seeds} & \textbf{Test seeds} \\
\midrule
IC1 & 1101, 1103, 1105, 1107 & 1102, 1104, 1106, 1108\\
IC2 & 1201, 1203, 1205, 1207 & 1202, 1204, 1206, 1208\\
IC3 & 1301, 1303, 1305, 1307 & 1302, 1304, 1306, 1308\\
IC4 & 1401, 1403, 1405, 1407 & 1402, 1404, 1406, 1408\\
IC5 & 1501, 1503, 1505, 1507 & 1502, 1504, 1506, 1508\\
\bottomrule
\end{tabular}
\end{table}

\subsection{Hyperparameters}\label{app:hparams}
Unless otherwise noted, we use the following defaults:
\[
\begin{aligned}
&\text{Stage I (PNS filter):} && \tau_{\mathrm{pns}}=0.8.\\[2pt]
&\text{Stage II (routing objective):} && \lambda_{\text{route}}=0.1,\ \tau_{\mathrm{cov}}=0.2,\ \tau_{\text{prune}}=10^{-2},\\
&&& K=80,\ \lambda_{\text{len}}=0.3,\ \epsilon_{\text{imp}}=10^{-6}.\\[2pt]
&\text{Hybrid baseline weight:} && \alpha=0.7\ \text{in } \alpha\,\overline{\mathrm{Cov}}+(1-\alpha)\,\overline{r}.\\[2pt]
&\text{Overall Perf weights:} && w_{\mathrm{pns}}=0.8,\ w_{\mathrm{cov}}=0.2,\ w_{\mathrm{gap}}=0.1.\\[2pt]
&\text{Stage III (factor selection):} && \theta_{\mathrm{dist}}=0.3,\ \alpha=0.05.\\[2pt]
&\text{RNG seeds:} && s=42.
\end{aligned}
\]

Environment train/val/test seeds for IC1--IC5 follow Table~\ref{tab:app-seeds} (odd for train, even for test).

\noindent\textit{Notes for Stage III.} For factor $f_k$, distance on the complier set is the (normalized) 1-Wasserstein metric:
\[
D_k = \frac{W_1(\widehat{F}_{k,\theta}, \widehat{F}_{k,\theta_0})}{s_k},
\]
where $s_k$ is a scale term.

A factor is declared \emph{key} only if both hold: (i) $D_k \ge \theta_{\text{dist}}$ with default $\theta_{\text{dist}}=0.3$, and (ii) a two-sided permutation test on $W_1$ is significant at $\alpha=0.05$ after Holm–Bonferroni correction across all factor$\times$group (and baseline, if multiple) tests within the route.

\section{Complexity and Overhead}
\label{app:complexity}

\paragraph{Notation.}
\[
\begin{aligned}
C&:\#\text{(IC,norm,base,task)},\quad M:\text{rows/CSV},\quad G:\text{groups},\\
R&:\text{rules (Stage I)},\quad B:\text{buckets},\quad U:\text{eval items},\quad L:\text{final list len.}
\end{aligned}
\]
\begin{table}[t]
\centering
\caption{Asymptotic complexity (Big-\(O\)).}
\label{tab:complexity}
\small
\begin{tabular}{@{}lll@{}}
\toprule
\textbf{Stage} & \textbf{Time} & \textbf{Space} \\
\midrule
Stage I (PNS) 
& \( \mathcal{O}(C\,M\,G) \)
& \( \mathcal{O}(1) \) \\
Stage II (routing) 
& \( \mathcal{O}\big(CG^{2}M + RC + R\log R + L R B U + L^{2} B U\big) \)
& \( \mathcal{O}(R B U) \) \\
Stage III (attribution) 
& \( \mathcal{O}(L U) \)
& \( \mathcal{O}(U) \) \\
\bottomrule
\end{tabular}
\end{table}

\paragraph{Stage II components.}
\begin{align*}
\text{Target-set selection} &:& \mathcal{O}(C\,G^{2}\,M) \\
\text{Scoring / filtering}  &:& \mathcal{O}(R\,C) + \mathcal{O}(R\log R) \\
\text{Greedy routing}       &:& \mathcal{O}(L\,R\,B\,U) \\
\text{Backward pruning}     &:& \mathcal{O}(L^{2}\,B\,U)
\end{align*}

\paragraph{Overall.}
Dominant terms are typically \( \mathcal{O}(C\,G^{2}\,M) \) and \( \mathcal{O}(L\,R\,B\,U) \); Stage~III is negligible relative to Stage~II.
\clearpage

\end{document}